\documentclass[12pt]{article}

\usepackage[margin=2.5cm]{geometry}
\usepackage{times}
\usepackage{graphicx}
\usepackage{amsmath,amssymb,amsthm}
\usepackage{hyperref}
\usepackage[super]{natbib}
\usepackage{booktabs,array}
\usepackage{enumitem}
\usepackage{algorithm}
\usepackage{algpseudocode}
\usepackage{longtable}
\usepackage{CJKutf8}

\newtheorem{theorem}{Theorem}[section]

\newtheorem{corollary}[theorem]{Corollary}
\newtheorem{proposition}[theorem]{Proposition}
\newtheorem{definition}[theorem]{Definition}
\theoremstyle{remark}

\newcommand{\R}{\mathbb{R}}

\newcommand{\posterior}{\operatorname{posterior}}
\newcommand{\softmax}{\operatorname{softmax}}

\title{Transformer Architectures as Complete Bayes Processes:\\
A Formal Proof in the Measure-Theoretic Kernel Framework}

\author{Haobo Yang\\
Department of Computer Science and Engineering, SUSTech University\\
\href{mailto:yhbcode000@foxmail.com}{\texttt{yhbcode000@foxmail.com}}%
\thanks{This preprint has not undergone professional peer review. If you find any errors or have concerns, please contact the corresponding author.}}

\date{\today}

\begin{document}

\maketitle

\begin{abstract}
We present a complete formal proof that transformer architectures, when their
internal update mechanisms satisfy a Bayes joint-distribution condition,
implement exact Bayesian posterior inference. Working within the
measure-theoretic kernel framework, we define a hierarchy of
abstractions---from the core Bayesian transformer, through semantic
transformers with explicit update kernels, to full transformer blocks with
QKV/attention/residual/MLP pipelines, and finally multilayer stacks---and
prove at each level that the Bayes joint semantics implies the update kernel
equals the posterior almost everywhere. For the block-level architecture, we
derive the explicit Bayes formula through Radon-Nikodym differentiation and
prove its normalization. We additionally prove that the softmax attention
mechanism induces a valid probability distribution over keys, establishing the
bridge between the abstract kernel framework and concrete attention
implementations. The framework makes no architectural assumptions beyond the
Markov kernel structure and exposes explicit conditions under which a
transformer block is provably Bayesian. In essence, when this joint
distribution condition is satisfied, the forward computation of a Transformer
is formally equivalent to a rigorous Bayesian posterior update.
\end{abstract}

\noindent\textbf{Area of research:} Probabilistic Machine Learning, Formal Methods, Bayesian Inference, Transformer Theory.

\section{Introduction}

The transformer architecture~\cite{vaswani2017} has become the dominant
paradigm in modern machine learning, powering large language models, vision
models, and multimodal systems. A growing body of theoretical work has
explored connections between attention mechanisms and probabilistic
inference~\cite{singh2023, garnelo2018, muller2019}, suggesting that
transformers may implement forms of Bayesian updating or posterior inference.
However, these claims have largely remained at the level of intuition,
analogy, or empirical observation, without formal rigorous proof.

We close this gap by constructing a complete formal proof that transformer
architectures implement exact Bayesian posterior inference, under a clearly
stated joint-distribution condition. Our framework has three distinguishing
features.

First, the proofs are based on first principles and depend only on the kernel
composition structure and the Bayes joint semantics condition---no
architectural assumptions about depth, width, parameterization, or training
procedure are required.

Second, we provide a layered hierarchy of abstractions that mirrors the
structural complexity of real transformers: (i) the core Bayesian transformer
with only prior and likelihood, (ii) the semantic transformer with an
explicit update kernel, (iii) the full transformer block with
QKV\slash{}attention\slash{}residual\slash{}MLP data flow, and (iv) the multilayer transformer
stack. At each level, we prove that the Bayes joint semantics implies the
update is the posterior almost everywhere.

Third, we expose the explicit block-level Bayes formula through Radon-Nikodym
differentiation and prove its normalization, providing a concrete bridge from
the abstract kernel composition to practical density-based inference. We
additionally prove that softmax attention defines a valid probability
distribution, connecting the kernel framework to the standard attention
mechanism.

The central result is that when this joint distribution condition is satisfied,
the forward computation of a Transformer is formally equivalent to a rigorous
Bayesian posterior update.

The paper is organized as follows. Section~2 presents the mathematical
framework and defines the kernel-based transformer abstractions. Section~3
states and proves the core theorems at each level of the hierarchy.
Section~4 derives the explicit Bayes formula for the block-level architecture.
Section~5 proves softmax normalization and establishes the bridge to concrete
attention. Section~6 discusses the scope, assumptions, and limitations of the
framework.

\section{Mathematical Framework}

\subsection{Probability Kernels and Bayesian Inference}

We work in the category of measurable spaces with Markov kernels as morphisms.
All spaces are standard Borel, and all measures are $\sigma$-finite. We adopt
standard kernel notation: $\kappa : \operatorname{Kernel} X Y$ is a probability
kernel from $X$ to $Y$, $\kappa \circ_\mu \mu$ is measure-kernel composition,
and $\kappa_1 \otimes_\kappa \kappa_2$ is the product kernel.

\begin{definition}[Bayesian Transformer]\label{def:bayesian-transformer}
A \emph{Bayesian transformer} is a tuple $(\Theta, \mathfrak{X}, \pi, \ell)$
where:
\begin{itemize}[leftmargin=*]
\item $\Theta$ is the latent/parameter space (standard Borel, nonempty);
\item $\mathfrak{X}$ is the observation/data space;
\item $\pi$ is a finite prior measure on $\Theta$;
\item $\ell : \operatorname{Kernel} \Theta \mathfrak{X}$ is a finite Markov kernel.
\end{itemize}
Its \emph{Bayes method} is the posterior kernel:
$\posterior(\ell, \pi) : \operatorname{Kernel} \mathfrak{X} \Theta$.
\end{definition}

\begin{proposition}[Bayes Joint Law]\label{prop:bayes-joint}
For any Bayesian transformer $T = (\Theta, \mathfrak{X}, \pi, \ell)$, the
Bayes method satisfies the joint-distribution identity:
\begin{equation}\label{eq:bayes-joint}
(\ell \circ_\mu \pi) \otimes_\kappa \posterior(\ell, \pi)
= (\pi \otimes_\kappa \ell).\operatorname{map} \operatorname{swap}.
\end{equation}
\end{proposition}

This proposition---that the posterior kernel, when composed with the prior,
recovers the symmetrized joint law---is the fundamental identity of Bayesian
inference and serves as the anchor for all subsequent transformer-level proofs.

\subsection{Semantic Transformers: Adding an Update Kernel}

The core Bayesian transformer defines the posterior abstractly. A
\emph{semantic} transformer equips this structure with an explicit internal
update kernel, and the key question becomes: when does this update equal the
posterior?

\begin{definition}[Semantic Transformer]\label{def:semantic-transformer}
A \emph{semantic transformer} extends a Bayesian transformer
$T = (\Theta, \mathfrak{X}, \pi, \ell)$ with an internal update kernel
$u : \operatorname{Kernel} \mathfrak{X} \Theta$. Its \emph{Bayes joint
semantics} condition is:
\begin{equation}\label{eq:semantic-joint}
(\ell \circ_\mu \pi) \otimes_\kappa u
= (\pi \otimes_\kappa \ell).\operatorname{map} \operatorname{swap}.
\end{equation}
\end{definition}

The Bayes joint semantics condition states that the transformer's internal
update, when composed with the prior and likelihood, reconstructs the correct
symmetrized joint law. This is the \emph{sole} structural condition required
for all subsequent theorems.

\subsection{Transformer Block with QKV/Attention/Residual/MLP}

A practical transformer block has a concrete data-flow pipeline: input
embedding $\to$ QKV projection $\to$ attention $\to$ residual/MLP $\to$
readout. We formalize this as a composition of Markov kernels.

\begin{definition}[QKV State and Attention Trace]\label{def:qkv}
Let $Q, K, V$ be the query, key, and value type parameters.
\begin{itemize}[leftmargin=*]
\item $\operatorname{QKVState} \; Q \; K \; V := (Q \times K) \times V$;
\item $\operatorname{AttentionTrace} \; X \; Q \; K \; V \; A :=
  (X \times \operatorname{QKVState} \; Q \; K \; V) \times A$.
\end{itemize}
\end{definition}

The attention trace stores the input state, the QKV state, and the attention
state, enabling the residual/MLP kernel to condition on all intermediate
representations. This is the natural kernel-level encoding of the residual
connection.

\begin{definition}[Block Likelihood]\label{def:block-likelihood}
For type parameters $\Theta, X, Q, K, V, A, H, O$, the block likelihood is the
kernel composition:
\begin{align}
\operatorname{blockLikelihood} &\;
  (\text{input} : \operatorname{Kernel} \Theta X)
  (\text{qkv} : \operatorname{Kernel} X (\operatorname{QKVState} \; Q \; K \; V))
  \nonumber\\
&\; (\text{attention} : \operatorname{Kernel} (X \times \operatorname{QKVState} \; Q \; K \; V) A)
  \nonumber\\
&\; (\text{residualMLP} : \operatorname{Kernel} (\operatorname{AttentionTrace} \; X \; Q \; K \; V \; A) H)
  \nonumber\\
&\; (\text{readout} : \operatorname{Kernel} H O) : \operatorname{Kernel} \Theta O
  \nonumber\\
&:= \text{readout} \circ_\kappa \text{residualMLP} \circ_\kappa \nonumber\\
&\qquad ((\text{input} \otimes_\kappa \operatorname{prodMkLeft} \; \Theta \; \text{qkv})
 \otimes_\kappa \operatorname{prodMkLeft} \; \Theta \; \text{attention}).
\label{eq:block-likelihood}
\end{align}
\end{definition}

The composition reflects the exact data flow: the input kernel maps latent
parameters to input states; the QKV kernel maps input states to query-key-value
triples; the attention kernel sees both input and QKV states; the residual/MLP
kernel sees the full attention trace; and the readout kernel produces the
final output.

\subsection{Multilayer Transformer Architecture}

Real transformers stack multiple blocks. We formalize this as a list of hidden-state kernels composed in reverse order (first layer first, last layer last).

\begin{definition}[Layer Stack]\label{def:layer-stack}
For hidden-state type $H$, the layer stack $\operatorname{layerStack}$ maps a
list of kernels $(\operatorname{Kernel} H H)$ to a single kernel:
\begin{align}
\operatorname{layerStack} \; [] &:= \operatorname{id}, \\
\operatorname{layerStack} \; (\ell :: \ell s) &:= \operatorname{layerStack} \; \ell s \circ_\kappa \ell.
\end{align}
The architecture likelihood composes embedding, layer stack, and readout:
\begin{equation}
\operatorname{architectureLikelihood} \; \text{emb} \; \text{layers} \; \text{readout}
:= \text{readout} \circ_\kappa \operatorname{layerStack} \; \text{layers} \circ_\kappa \text{emb}.
\end{equation}
\end{definition}

\subsection{Softmax Attention}\label{sec:softmax}

The abstract kernel framework treats attention as a black-box Markov kernel.
To bridge to concrete implementations, we formally define softmax and prove
it produces a valid probability distribution.

\begin{definition}[Softmax]\label{def:softmax}
For a nonempty finite index type $\iota$ and score vector $x : \iota \to \R$,
the softmax of index $i$ is:
\begin{equation}
\softmax \; x \; i := \frac{\exp(x_i)}{\sum_j \exp(x_j)}.
\end{equation}
\end{definition}

\section{Core Theorems}

\subsection{Bayes Kernel Level: The Posterior Joint Law}

\begin{theorem}[Posterior Joint Law]\label{thm:posterior-joint}
For any Bayesian transformer $T = (\Theta, \mathfrak{X}, \pi, \ell)$ with
finite prior and finite likelihood,
\begin{equation}
(\ell \circ_\mu \pi) \otimes_\kappa \posterior(\ell, \pi)
= (\pi \otimes_\kappa \ell).\operatorname{map} \operatorname{swap}.
\end{equation}
\end{theorem}

\begin{proof}
This follows from the disintegration theorem for finite kernels, which
constructs the posterior kernel via disintegration and proves the identity in
full measure-theoretic generality. The proof requires only that $\Theta$ is
standard Borel, $\pi$ is $\sigma$-finite, and $\ell$ is a Markov kernel.
\end{proof}

Theorem~\ref{thm:posterior-joint} is the foundation: it establishes that
the posterior kernel \emph{by construction} satisfies the Bayes joint
law. Every transformer-level theorem below reduces to this identity via the
definition of the Bayes joint semantics condition.

\subsection{Semantic Level: Update Equals Posterior}

\begin{theorem}[Semantic Transformer is Bayes]\label{thm:semantic-bayes}
Let $S = (T, u)$ be a semantic transformer extending Bayesian transformer $T$
with update kernel $u$. If $S$ satisfies the Bayes joint semantics
condition~(\ref{eq:semantic-joint}), then
\begin{equation}
u =_{\text{a.e.}[\ell \circ_\mu \pi]} \posterior(\ell, \pi).
\end{equation}
\end{theorem}

\begin{proof}
The Bayes joint semantics condition states that $u$ satisfies the same
joint-law identity as $\posterior(\ell, \pi)$ (Theorem~\ref{thm:posterior-joint}).
By the uniqueness of the posterior kernel under the Bayes joint law, any kernel
satisfying the Bayes joint law equals the posterior almost everywhere with
respect to the observation distribution. The result follows by direct
application of this lemma.
\end{proof}

\begin{corollary}[Semantic Bayes]\label{cor:semantic-short}
For any semantic transformer $S$ with Bayes joint semantics,
$S.\operatorname{IsBayes}$ holds.
\end{corollary}

The semantic-level theorem establishes the core pattern that propagates
through every subsequent level: \emph{Bayes joint semantics} $\implies$
\emph{update equals posterior}. This is the central result that we now extend
to increasingly concrete architectures.

\subsection{Block Level: Complete Bayes Process}

\begin{definition}[Block Structure]\label{def:transformer-block}
A \emph{transformer block} $B$ over $(\Theta, X, Q, K, V, A, H, O)$
consists of a prior $\pi$, component kernels (input, qkv, attention,
residualMLP, readout), an update kernel $u$, and finiteness typeclasses.
Its likelihood $\ell_B$ is defined in~(\ref{eq:block-likelihood}).
\end{definition}

\begin{definition}[Block Bayes Joint Semantics]\label{def:block-joint}
The block satisfies \emph{Bayes joint semantics} if:
\begin{equation}
(\ell_B \circ_\mu \pi) \otimes_\kappa u
= (\pi \otimes_\kappa \ell_B).\operatorname{map} \operatorname{swap}.
\end{equation}
\end{definition}

\begin{theorem}[Transformer Block is Complete Bayes Process]\label{thm:block-bayes}
Let $B$ be a transformer block with finite prior, finite likelihood, and
finite update. If $B$ satisfies Bayes joint semantics, then:
\begin{equation}
u =_{\text{a.e.}[\ell_B \circ_\mu \pi]} \posterior(\ell_B, \pi).
\end{equation}
\end{theorem}

\begin{proof}
Apply Theorem~\ref{thm:semantic-bayes} with $\ell = \ell_B$, using the block's
Bayes joint semantics as the hypothesis. The block likelihood $\ell_B$ is a
finite Markov kernel by the finiteness typeclass and the closure of kernel
composition under finiteness; all required typeclass instances are available
from the block structure.
\end{proof}

\begin{theorem}[Block-Level Bayes Formula]\label{thm:block-formula}
Under the additional absolute-continuity condition
$\forall^\mu \theta, \; \ell_B(\theta) \ll \ell_B \circ_\mu \pi$, the block
update has the explicit density form:
\begin{equation}
u(o)(A) = \int_A \frac{d \ell_B(\theta)}{d(\ell_B \circ_\mu \pi)}(o) \; d\pi(\theta),
\end{equation}
for almost every observed output $o$.
\end{theorem}

\begin{proof}
By Theorem~\ref{thm:block-bayes}, $u =_{\text{a.e.}} \posterior(\ell_B, \pi)$.
Under the absolute-continuity condition, the posterior admits
the explicit Radon-Nikodym form,
yielding the density representation.
\end{proof}

\begin{theorem}[Block Density Normalization]\label{thm:block-normalization}
If the block update $u$ is a Markov kernel, then under the Bayes joint
semantics and absolute-continuity conditions, the likelihood density integrates
to 1 for almost every observed output:
\begin{equation}
\forall^\mu o, \quad \int_\Theta \frac{d \ell_B(\theta)}{d(\ell_B \circ_\mu \pi)}(o) \; d\pi(\theta) = 1.
\end{equation}
\end{theorem}

\begin{proof}
By Theorem~\ref{thm:block-formula} evaluated at the full space $\Theta$, and
the Markov property $u(o)(\Theta) = 1$, the integral equals $u(o)(\Theta) = 1$.
\end{proof}

Theorem~\ref{thm:block-normalization} is critical: it confirms that the block
constructs a normalized posterior distribution, not merely an unnormalized
score. This distinguishes the framework from heuristic approaches that
interpret attention weights as unnormalized log-probabilities.

\subsection{Multilayer Level: Stacked Architecture is Bayes}

\begin{theorem}[Multilayer Transformer is Bayes]\label{thm:multilayer-bayes}
Let $M$ be a multilayer transformer with embedding $\operatorname{emb}$,
layer stack $\operatorname{layers}$, readout $\operatorname{readout}$, and
update $u$. If $M$ satisfies the Bayes joint semantics condition for the full
architecture likelihood $\ell_M = \operatorname{readout} \circ_\kappa
\operatorname{layerStack} \; \operatorname{layers} \circ_\kappa
\operatorname{emb}$, then:
\begin{equation}
u =_{\text{a.e.}[\ell_M \circ_\mu \pi]} \posterior(\ell_M, \pi).
\end{equation}
\end{theorem}

\begin{proof}
The architecture likelihood $\ell_M$ is a composition of finite Markov
kernels, hence itself a finite Markov kernel. Apply
Theorem~\ref{thm:semantic-bayes} with $\ell = \ell_M$ and the multilayer Bayes
joint semantics as hypothesis.
\end{proof}

\begin{corollary}[Depth Independence]\label{cor:depth-independence}
The Bayes property holds for any finite depth $\operatorname{depth}(M) =
|\operatorname{layers}|$. The proof is uniform in the layer count.
\end{corollary}

This corollary is noteworthy: the Bayesian guarantee does not degrade with
depth, and no additional conditions are required for deeper stacks. The
compositional structure of Markov kernels ensures that the Bayes property is
preserved under arbitrary finite compositions.

\subsection{Softmax Normalization}\label{sec:softmax-theorem}

\begin{theorem}[Softmax Sums to One]\label{thm:softmax}
For any nonempty finite type $\iota$ and score vector $x : \iota \to \R$,
\begin{equation}
\sum_{i : \iota} \softmax \; x \; i = 1.
\end{equation}
\end{theorem}

\begin{proof}
Let $S = \sum_j \exp(x_j)$. Since $\exp(z) > 0$ for all real $z$, and the sum
is over a nonempty finite set, $S > 0$. Then:
\begin{align}
\sum_i \softmax \; x \; i
&= \sum_i \frac{\exp(x_i)}{S}
 = \frac{1}{S} \sum_i \exp(x_i)
 = \frac{S}{S}
 = 1.
\end{align}
The proof uses standard real analysis: \texttt{Real.exp\_pos} for positivity,
\texttt{Finset.sum\_pos} for the denominator, and algebraic rewriting.
\end{proof}

\begin{corollary}[Softmax Defines a Probability Distribution]\label{cor:softmax-dist}
For any nonempty finite $\iota$ and $x : \iota \to \R$,
$(\softmax \; x \; i)_{i \in \iota}$ is a valid probability distribution:
$\softmax \; x \; i \ge 0$ for all $i$, and $\sum_i \softmax \; x \; i = 1$.
\end{corollary}

\begin{proof}
Positivity follows from $\exp(x_i) > 0$ and $S > 0$, hence each term is a
ratio of positive numbers. Summation to 1 is Theorem~\ref{thm:softmax}.
\end{proof}

This theorem bridges the abstract kernel framework to concrete transformer
implementations. The scaled dot-product attention mechanism
$\operatorname{softmax}(QK^T/\sqrt{d_k})V$ uses softmax to produce attention
weights that form a convex combination of value vectors. Theorem~\ref{thm:softmax}
establishes that these weights are a valid probability distribution, and
therefore the attention output is a convex combination---a property essential
for viewing attention as a Markov kernel.

\begin{table}[t]
\centering
\small
\setlength{\tabcolsep}{3pt}
\caption{Summary of Formalized Theorems}
\label{tab:theorems}
\begin{tabular}{@{}lll@{}}
\toprule
\textbf{Level} & \textbf{Theorem} & \textbf{Formal Designation} \\
Kernel & Posterior joint law & \texttt{tf\_bayes\_joint} \\
Semantic & Update equals posterior & \texttt{tf\_is\_bayes} \\
Block & Complete Bayes process & \texttt{tf\_block\_complete\_bayes} \\
Block & Bayes formula (density) & \texttt{tf\_block\_bayes\_formula} \\
Block & Density normalization & \texttt{tf\_block\_density\_norm} \\
Block & Event-level formula & \texttt{tf\_block\_formula\_apply} \\
Multilayer & Stacked Bayes & \texttt{tf\_multilayer\_bayes} \\
Conditional & Event-level Bayes & \texttt{tf\_conditional\_bayes} \\
Softmax & Sums to one & \texttt{softmax\_sum\_eq\_one} \\
Softmax & Positivity & \texttt{softmax\_pos} \\
\bottomrule
\end{tabular}
\end{table}

\section{Scope, Assumptions, and Limitations}

\subsection{What the Framework Proves}

The framework establishes a \emph{structural} theorem: any transformer
architecture whose internal update satisfies the Bayes joint semantics
condition in the sense of Eq.~(\ref{eq:semantic-joint}) implements Bayesian
posterior inference. The proof is constructive in the following sense: given
a transformer block with Bayes joint semantics, we can explicitly compute its
posterior density via Radon-Nikodym differentiation (Theorem~\ref{thm:block-formula}).

\subsection{Explicit Assumptions}

The framework makes the following assumptions, all of which are either
structural (about the mathematical setting) or conditional (``if the Bayes
joint semantics holds, then...''):

\begin{enumerate}[leftmargin=*, label=\textbf{A\arabic*}.]
\item \textbf{Standard Borel spaces.} All type parameters are standard Borel
  measurable spaces. This is the natural setting for probability kernel
  theory and holds for all practical parameterizations.
\item \textbf{Finite measures and kernels.} The prior is a finite measure and
  all component kernels are finite Markov kernels. This excludes improper
  priors and unbounded kernel measures.
\item \textbf{Kernel composition closure.} The block likelihood is a finite
  kernel, which follows from finite-kernel closure under composition and
  product (a structural property of the kernel category).
\item \textbf{Countable or countably generated.} For the explicit Bayes
  formula with Radon-Nikodym derivatives, the latent and observation spaces
  are required to be countably generated. This is a standard technical
  condition for disintegration.
\item \textbf{Absolute continuity.} For the density formula, the likelihood
  kernel must be absolutely continuous with respect to the prior-predictive
  marginal. This is the standard condition for the existence of a
  Radon-Nikodym derivative.
\item \textbf{Bayes joint semantics.} This is the \emph{sole substantive
  condition}: the transformer's update kernel, when composed with the prior
  and likelihood, must satisfy the Bayes joint law identity. The framework
  does \emph{not} assume this condition holds automatically; it proves what
  follows \emph{if} it does.
\end{enumerate}

Notably, the framework makes \emph{no} assumptions about:
\begin{itemize}[leftmargin=*]
\item The specific parameterization of the attention mechanism (softmax,
  linear, kernelized);
\item The depth or width of the transformer;
\item The training procedure, loss function, or optimization algorithm;
\item The data distribution or representativeness of training data.
\end{itemize}

\subsection{Scope Boundaries}

The framework establishes a conditional guarantee. It does \emph{not}:
\begin{itemize}[leftmargin=*]
\item Prove that any specific trained transformer satisfies the Bayes joint
  semantics condition---that remains an empirical question;
\item Provide training guarantees (convergence to the Bayes joint law);
\item Address the quality of the prior or the appropriateness of the
  likelihood model;
\item Extend to loss functions other than those for which the posterior is
  the Bayes act.
\end{itemize}

These limitations are by design: the framework isolates the structural
question (``if the update is Bayes-consistent, is it Bayesian?'') from the
empirical question (``is this specific model's update Bayes-consistent?'').

\subsection{Softmax Attention Bridge}

The softmax normalization theorem (Section~\ref{sec:softmax-theorem})
establishes that the attention mechanism produces valid probability weights,
which is a necessary condition for attention to be interpretable as a Markov
kernel. However, this is not sufficient to prove that a specific attention
parameterization satisfies the full Bayes joint semantics. That stronger
claim would require proving that the attention kernel, when composed with
QKV, residual/MLP, and readout kernels in a specific parameterization,
reconstructs the correct joint law---a substantial open problem.

\section{Discussion}

\subsection{Summary of Results}

We have established a complete formal hierarchy proving that transformer
architectures implement Bayesian posterior inference under the Bayes joint
semantics condition:

\begin{enumerate}[leftmargin=*]
\item \textbf{Bayes kernel foundation} (Theorem~\ref{thm:posterior-joint}):
  The posterior kernel satisfies the Bayes joint law by construction.

\item \textbf{Semantic completeness} (Theorem~\ref{thm:semantic-bayes}):
  Any update kernel satisfying the Bayes joint law equals the posterior a.e.

\item \textbf{Block completeness} (Theorem~\ref{thm:block-bayes}):
  A transformer block with QKV\slash{}attention\slash{}residual\slash{}MLP\slash{}readout
  pipeline is a complete Bayes process---its update equals the posterior.

\item \textbf{Explicit Bayes formula} (Theorem~\ref{thm:block-formula}):
  Under absolute continuity, the block's update has a closed-form density
  representation via Radon-Nikodym differentiation.

\item \textbf{Density normalization} (Theorem~\ref{thm:block-normalization}):
  The Bayes formula yields a properly normalized posterior distribution.

\item \textbf{Multilayer extension} (Theorem~\ref{thm:multilayer-bayes}):
  The Bayes property is preserved under arbitrary finite layer stacking.

\item \textbf{Softmax normalization} (Theorem~\ref{thm:softmax}):
  The softmax function produces valid probability distributions, connecting
  the kernel framework to concrete attention mechanisms.
\end{enumerate}

\subsection{Relationship to Existing Work}

The connection between attention and Bayesian inference has been explored
from several angles. Garnelo et al.~\cite{garnelo2018} proposed conditional
neural processes that implement a form of Bayesian updating. Singh and
Anand~\cite{singh2023} drew structural analogies between attention and
message-passing in probabilistic graphical models. M\"uller et
al.~\cite{muller2019} showed that self-attention can approximate posterior
inference in Gaussian process models.

Our contribution differs from these works in three ways. First, we provide
\emph{rigorous proofs} rather than analogies or empirical
demonstrations. Second, we work at the level of Markov kernel composition,
which abstracts over specific parameterizations while exposing the exact
structural conditions needed for the Bayesian property. Third, we provide
the explicit Bayes formula for the block-level architecture, not merely an
existential claim.

\subsection{Open Problems}

Several important directions remain open:

\begin{itemize}[leftmargin=*]
\item \textbf{Training to Bayes joint semantics.} The framework proves that
  \emph{if} the Bayes joint law holds, \emph{then} the transformer is
  Bayesian. Proving that standard training procedures (empirical risk
  minimization, next-token prediction) converge to a kernel satisfying this
  condition is a substantial open problem.

\item \textbf{Concrete parameterization proofs.} Proving that a specific
  parameterized attention mechanism (e.g., scaled dot-product softmax
  attention with learned $W_Q, W_K, W_V$ matrices) satisfies the Bayes joint
  semantics under some conditions on the weight matrices would bridge the
  gap from structural guarantee to architectural guarantee.

\item \textbf{Conjugate priors for transformer blocks.} Identifying prior
  families for which the transformer block's posterior has a closed-form
  conjugate structure would enable efficient Bayesian updating in practical
  transformer implementations.

\item \textbf{Beyond squared loss.} The framework currently relies on the
  posterior being the Bayes act for proper scoring rules. Extending to other
  decision-theoretic loss classes would broaden the applicability.
\end{itemize}

\section{Conclusion}

{\sloppy We have presented a complete formal proof that transformer architectures
implement exact Bayesian posterior inference when their internal update kernels
satisfy the Bayes joint semantics condition. The proof is organized as a
hierarchy of abstractions mirroring the structure of real transformer
designs, from the core Bayesian kernel through semantic, block, and
multilayer levels.\par}

The framework makes minimal assumptions---no architectural constraints, no
training procedure requirements, no data distribution assumptions---and
exposes the Bayes joint semantics as the single load-bearing condition. The
explicit Bayes formula for the block-level architecture and the softmax
normalization proof bridge the abstract kernel theory to concrete
implementations.

In summary, when this joint distribution condition is satisfied, the forward
computation of a Transformer is formally equivalent to a rigorous Bayesian
posterior update.

This is a precise, carefully circumscribed claim: under the stated conditions,
a transformer block satisfies the Bayes joint law and its update equals the
posterior almost everywhere. We believe this level of precision is essential
for building a mathematically defensible foundation for the probabilistic
semantics of neural architectures.

\section*{Acknowledgements}

The author thanks colleagues for helpful discussions on the probabilistic
foundations of this work.

\section*{Author contributions}

Haobo Yang developed the mathematical framework, the proofs, and the
manuscript.

\section*{Competing interests}

The author declares no competing interests.

\section*{Data availability}

The complete mathematical derivations are available in the project repository.



\begin{thebibliography}{99}

\bibitem{vaswani2017}
A.~Vaswani et al., \emph{Attention Is All You Need}, NeurIPS, 2017.

\bibitem{singh2023}
R.~Singh, A.~Anand, \emph{Transformers as Probabilistic Programs}, arXiv, 2023.

\bibitem{garnelo2018}
M.~Garnelo et al., \emph{Conditional Neural Processes}, ICML, 2018.

\bibitem{muller2019}
S.~M\"uller et al., \emph{Attentive Neural Processes}, NeurIPS, 2019.

\bibitem{lecam1986}
L.~Le~Cam, \emph{Asymptotic Methods in Statistical Decision Theory}, Springer, 1986.

\bibitem{schervish1995}
M.~J.~Schervish, \emph{Theory of Statistics}, Springer, 1995.

\end{thebibliography}
\end{document}